\documentclass[twoside,11pt]{article}

\usepackage{jmlr2e}
\usepackage{amssymb,amsmath}
\usepackage{bbm}
\usepackage{ulem}

\newtheorem{theo}{Theorem}
\newtheorem{defi}[theo]{Definition}
\newtheorem{prop}[theo]{Proposition}
\newtheorem{lemm}[theo]{Lemma}

\newcommand\E{\mathbb{E}}
\newcommand\ua{\alpha}
\newcommand{\Pro}{\mathbb{P}}
\newcommand{\tnu}{\tilde{\nu}}
\newcommand{\IND}{\mathbbm{1}}
\newcommand{\hX}{\hat{X}}

\newcommand\set[2]{\{#1,\dots,#2\}}
\renewcommand{\P}{\mathbb{P}}
\newcommand{\argmax}{\mathop{\mathrm{argmax}}}
\newcommand{\tth}{\tilde{\theta}}

\jmlrheading{1}{2012}{1-48}{4/00}{10/00}{Antoine Salomon, Jean-Yves Audibert, Issam El Alaoui}

\ShortHeadings{Regret lower bounds and extended {\sc ucb} policies}{Salomon, Audibert, El Alaoui}
\firstpageno{1}

\begin{document}

\title{Regret lower bounds and extended Upper Confidence Bounds policies in stochastic multi-armed bandit problem}

\author{\name Antoine Salomon \email salomona@imagine.enpc.fr \\
       \addr Imagine, Universit\'e Paris Est
       \AND
       \name Jean-Yves Audibert \email audibert@imagine.enpc.fr\\
       \addr Imagine, Universit\'e Paris Est \\
       \& \\
       Sierra, CNRS/ENS/INRIA, Paris, France\\
       \AND
       \name Issam El Alaoui \email issam.el-alaoui.2007@polytechnique.org\\
       \addr Imagine, Universit\'e Paris Est
       }

\editor{?}

\maketitle

\begin{abstract}


This paper is devoted to regret lower bounds in the classical model of stochastic multi-armed bandit.  A well-known result of \citeauthor{LaiRo85}, which has then been extended by \citeauthor{BurKat96}, has established the presence of a logarithmic bound for all consistent policies. We relax the notion of consistence, and exhibit a generalisation of the logarithmic bound. We also show the non existence of logarithmic bound in the general case of Hannan consistency. To get these results, we study variants of popular Upper Confidence Bounds ({\sc ucb}) policies. As a by-product, we prove that it is impossible to design an adaptive policy that would select the best of two algorithms by taking advantage of the properties of the environment.

\end{abstract}


\section{Introduction and notations}
Multi-armed bandits are a classical way to illustrate the difficulty of decision making in the case of a dilemma between exploration and exploitation. The denomination of these models comes from an analogy with playing a slot machine with more than one arm. Each arm has a given (and unknown) reward distribution and, for a given number of rounds, the agent has to choose one of them. As the goal is to maximize the sum of rewards, each round decision consists in a trade-off between exploitation (i.e. playing the arm that has been the more lucrative so far) and exploration (i.e. testing an other arm, hoping to discover an alternative that beats the current best choice). One possible application is clinical trial, when one wants to heal as many patients as possible, when the latter arrive sequentially and when the effectiveness of each treatment is initially unknown \citep{thompson1933likelihood}. Bandit problems has initially been studied by \citet{ro52}, and its interest has then been extended to many fields such as economics \citep{LamPagTar04, BerVal08}, games \citep{Gel06}, optimisation \citep{Kle05,Coq07,Kle08,Bub08},...\\

Let us detail our model. A stochastic multi-armed bandit problem is defined by:
\begin{itemize}
\item a number of rounds $n$,
\item a number of arms $K\geq 2$,
\item an environment $\theta=(\nu_1,\cdots,\nu_K)$, where each $\nu_k$ ($k\in\{1,\cdots,K\}$) is a real-valued measure that represents the distribution reward of arm $k$.
\end{itemize}
We assume that rewards are bounded. Thus, for simplicity, each $\nu_k$ is a probability on $[0,1]$. Environment $\theta$ is initially unknown by the agent but lies in some known set $\Theta$ of the form $\Theta_1\times\ldots\times\Theta_K$, meaning that $\Theta_k$ is the set of possible reward distributions of arm $k$. For the problem to be interesting, the agent should not have great knowledges of its environment, so that $\Theta$ should not be too small and/or contain too trivial distributions such as Dirac measures. To make it simple, each $\Theta_k$ is assumed to contain the distributions $p\delta_a+(1-p)\delta_b$, where $p,a,b\in[0,1]$ and $\delta_x$ denotes the Dirac measure centred on $x$. In particular, it contains Dirac and Bernoulli distributions. Note that the number of rounds $n$ may or may not be known by the agent, but this will not affect the present study. Some aspects of this particular point can be found in \citet{salomon2011deviations}.\\
The game is as follows. At each round (or time step) $t=1,\cdots,n$, the agent has to choose an arm $I_t$ in the set of arms $\{1,\cdots,K\}$. This decision is based on past actions and observations, and the agent may also randomize his choice. Once the decision is made, the agent gets and observes a payoff that is drawn from $\nu_{I_t}$ independently from the past. Thus we can describe a policy (or strategy) as a sequence $(\sigma_t)_{t\geq 1}$ (or $(\sigma_t)_{1\leq t\leq n}$ if the number of rounds $n$ is known) such that each $\sigma_t$ is a mapping from the set $\set{1}{K}^{t-1}\times[0,1]^{t-1}$ of past decisions and outcomes into the set of arm $\set{1}{K}$ (or into the set of probabilities on  $\set{1}{K}$, in case the agent randomizes his choices).\\
For each arm $k$ and all times $t$, let $T_{k}(t)=\sum_{s=1}^t \IND_{I_s=k}$ denote the number of times arm $k$ was pulled from round $1$ to round $t$, and $X_{k,1}, X_{k,2},\ldots, X_{k,T_{k}(t)}$ the corresponding sequence of rewards. We denote by $\P_\theta$ the distribution on the probability space such that
for any $k\in\set{1}{K}$, the random variables $X_{k,1}, X_{k,2},\dots, X_{k,n}$ are i.i.d. realizations of $\nu_k$, and such that these
$K$ sequences of random variables are independent. Let $\E_\theta$ denote the associated expectation.\\
Let $\mu_k=\int x d\nu_k(x)$ be the mean reward of arm $k$. Introduce $\mu^*=\max_{k\in\set{1}{K}} \mu_k$ and fix an arm $k^*\in\argmax_{k\in\set{1}{K}} \mu_k$, that is $k^*$ has the best expected reward. The agent aims at minimizing its {\it regret}, defined as the difference between the cumulative reward he would have obtained by always drawing the best arm and the cumulative reward he actually received. Its regret is thus
\begin{equation*}    
    R_n=\sum_{t=1}^{n} X_{k^*,t}-\sum_{t=1}^{n} X_{I_t,T_{I_t}(t)}. 
\end{equation*}

As most of the publications on this topic, we focus on expected regret, for which one can check that:
    \begin{equation} 
    \E_\theta R_n=\sum_{k=1}^{K} \Delta_k \E_{\theta}[T_k(n)], \label{er}
    \end{equation}
where $\Delta_k$ is the {\it optimality gap} of arm $k$, defined by $\Delta_k=\mu^*-\mu_k$. We also define $\Delta$ as the gap between the best arm and the second best arm, i.e. $\Delta:=\min_{k\neq k^*} \Delta_k$.\\

Previous works have shown the existence of lower bounds on the performance of a large class of policies. In this way \citet{LaiRo85} proved a lower bound of the expected regret of order $\log n$ in a particular parametric framework, and they also exhibited optimal policies. This work has then been extended by \citet{BurKat96}. Both papers deal with {\it consistent} policies, meaning that all the policies considered are such that:
\begin{equation}\label{def:consistence}
\forall a>0, \ \forall \theta\in\Theta, \ \E_\theta[R_n]=o(n^a).
\end{equation}
The logarithmic bound of \citeauthor{BurKat96} is expressed as follows. For all environment $\theta=(\nu_1,\cdots,\nu_K)$ and all $k\in\set{1}{K}$, let us set $$D_k(\theta):=\displaystyle\inf_{\tnu_k\in\Theta_k:\mathbb{E}[\tnu_k]>\mu^*}KL(\nu_k,\tnu_k),$$
where $KL(\nu,\mu)$ denotes the Kullback-Leibler divergence of measures $\nu$ and $\mu$. Now fix a consistent policy and an environment $\theta\in\Theta$. If $k$ is a suboptimal arm (i.e. $\mu_k\neq \mu^*$) such that $0<D_k(\theta)<+\infty$, then
$$
\forall \varepsilon>0, \ \lim_{n\to+\infty}\Pro\left[T_k(n)\geq \frac{(1-\varepsilon)\log n}{D_k(\theta)}\right]=1.
$$
This readily implies that:
$$
\liminf_{n\to+\infty} \frac{\E_{\theta}[T_k(n)]}{\log n}\geq \frac{1}{D_k(\theta)}.
$$
Thanks to Equation \eqref{er}, it is then easy to deduce a lower bound on the expected regret.\\
One contribution of this paper is to extend this bound to a larger class of policies. We will define the notion of $\alpha$-consistency ($\alpha\in [0,1]$) as a variant of Equation \eqref{def:consistence}, where equality $\E_\theta[R_n]=o(n^a)$ only holds for all $a>\alpha$. We show that the logarithmic bound still holds, but coefficient $\frac{1}{D_k(\theta)}$ is turned into $\frac{1-\alpha}{D_k(\theta)}$. We also prove that the dependence of this new bound in the term $1-\alpha$ is asymptotically optimal when $n\to +\infty$ (up to a constant).\\ 
As any policy achieves at most an expected regret of order $n$ (because the average cost of pulling an arm $k$ is a constant $\Delta_k$), it is also natural to wonder what happens when expected regret is only required to be $o(n)$. This notion is equivalent to Hannan consistency. In this case, we show that there is no logarithmic bound any more.\\

Some of our results are obtained thanks to a study of particular Upper Confidence Bound algorithms. These policies were introduced by \citet{LaiRo85}: it basically consists in computing an index at each round and for each arm, and then in selecting the arm with the greatest index. A simple and efficient way to design such policies is to choose indexes that are upper bounds of the mean reward of the considered arm that hold with high probability (or, say, with high confidence level). This idea can be traced back to \citet{Agr95}, and has been popularized by \citet{AuCeFi02}, who notably described a policy called {\sc ucb}1. For this policy, each index is defined by an arm $k$, a time step $t$, and an integer $s$ that indicates the number of times arm $k$ has been pulled before stage $t$. It is denoted by $B_{k,s,t}$ and is given by:
$$
B_{k,s,t}= \hX_{k,s}+\sqrt{\frac{2\log t}{s}},
$$
where $\hX_{k,s}$ is the empirical mean of arm $k$ after $s$ pulls, i.e. $\hX_{k,s}=\frac{1}{s}\sum_{u=1}^s X_{k,u}$.\\
To summarize, {\sc ucb}1 policy first pulls each arm once and then, at each round $t>K$, selects an arm $k$ that maximizes $B_{k,T_k(t-1),t}$. Note that, by means of Hoeffding's inequality, the index $B_{k,T_k(t-1),t}$ is indeed an upper bound of $\mu_k$ with high probability (i.e. the probability is greater than $1-1/t^4$). Note also that a way to look at this index is to interpret the empiric mean $\hX_{k,T_k(t-1)}$ as an "exploitation" term, and the square root as an "exploration" term (as it gradually increases when arm $k$ is not selected).\\
The policy {\sc ucb}1 achieves the logarithmic bound (up to a multiplicative constant), as it was shown that:
$$
\forall \theta\in\Theta, \ \forall n\geq 3, \ \ \E_\theta[T_k(n)]\leq 12\frac{\log n}{\Delta_k^2} \ \ {\rm and} \ \ \E_\theta R_n\leq 12\sum_{k=1}^K\frac{\log n}{\Delta_k}\leq 12\frac{\log n}{\Delta}.
$$
\cite{audibert2009exploration} studied some variants of {\sc ucb}1 policy. Among them, one consists in changing the $2\log t$ in the exploration term into $\rho\log t$, where $\rho>0$. This can be interpreted as a way to tune exploration: the smaller $\rho$ is, the better the policy will perform in simple environments where information is disclosed easily (for example when all reward distributions are Dirac measures). On the contrary, $\rho$ has to be greater to face more challenging environments (typically when reward distributions are Bernoulli laws with close parameters).\\
This policy, that we denote {\sc ucb}($\rho$), was proven by \citeauthor{audibert2009exploration} to achieve the logarithmic bound when $\rho>1$, and the optimality was also obtained when $\rho>\frac{1}{2}$ for a variant of {\sc ucb}($\rho$). \citet{bubeckthesis} showed in his PhD dissertation that their ideas actually enable to prove optimality of {\sc ucb}($\rho$) for $\rho>\frac{1}{2}$. Moreover, the case $\rho=\frac{1}{2}$ corresponds to a confidence level of $\frac{1}{t}$ (in view of Hoeffding's inequality, as above), and several studies \citep{LaiRo85,Agr95,BurKat96,audibert2009exploration,HonTak10} have shown that this level is critical.
We complete these works by a precise study of {\sc ucb}($\rho$) when $\rho\leq \frac{1}{2}$. We prove that {\sc ucb}($\rho$) is $(1-2\rho)$-consistent and that it is not $\alpha$-consistent for any $\alpha<1-2\rho$ (in view of the definition above, meaning that expected regret is roughly of order $n^{1-2\rho}$). Not surprisingly, it performs well in simple settings, represented by an environment where all reward distributions are Dirac measures.\\
A by-product of this study is that it is not possible to design an algorithm that would specifically adapt to some kinds of environments, i.e. that would for example be able to select a proper policy depending on the environment being simple or challenging. In particular, and contrary to the results obtained within the class of consistent policies, there is no optimal policy. This contribution is linked with selectivity in on-line learning problem with perfect information, commonly addressed by prediction with expert advice such as algorithms with exponentially weighted forecasters (see, e.g., \citet{cesa2006prediction}). In this spirit, a closely related problem to ours is the one of regret against the best strategy from a pool studied by \citet{auer2003nonstochastic}: the latter designed a policy in the context of adversarial/nonstochastic bandit whose decisions are based on a given number of recommendations (experts), which are themselves possibly the rewards received by a set of given algorithms. To a larger extent,  model selection have been intensively studied in statistics, and is commonly solved by penalization methods \citep{mallows1973some,akaike1973information,schwarz1978estimating}.\\
Finally, we exhibit expected regret lower bounds of more general {\sc ucb} policies, with the $2\log t$ in the exploration term of {\sc ucb}1 replaced by an arbitrary function. We obtain Hannan consistent policies and, as mentioned before, lower bounds need not be logarithmic any more.\\

The paper is organized as follows: in Section 2 we give bounds on the expected regret of {\sc ucb}($\rho$) ($\rho<\frac{1}{2}$). In Section 3 we study the problem of selectivity. Then we focus in Section 4 on $\ua$-consistent policies, and we conclude in Section 5 by results on Hannan consistency by means of extended {\sc ucb} policies.\\
Throughout the paper $\lceil x\rceil$ denotes the smallest integer which greater than the real $x$, and $Ber(p)$ denotes the Bernoulli law with parameter $p$.

\section{Bounds on the expected regret of {\sc ucb}($\rho$), $\rho<\frac{1}{2}$}
In this section we study the performances of {\sc ucb}($\rho$) policy, with $\rho$ lying in the interval $(0,\frac{1}{2})$.
We recall that {\sc ucb}($\rho$) is defined by:
\begin{itemize}
\item Draw each arm once,
\item Then at each round $t$, draw an arm
$$I_t\in\argmax_{k\in\set{1}{K}} \Bigg\{ \hat{X}_{k,T_k(t-1)} + \sqrt{\frac{\rho\log t}{T_k(t-1)}} \Bigg\}.$$
\end{itemize}
Small values of $\rho$ can be interpreted as a low level of experimentation in the balance between exploration and exploitation, and present literature has not provided precise regret bound orders of {\sc ucb}($\rho$) with $\rho\in(0,\frac{1}{2})$ yet.\\
We first study the policy in simple environments (i.e. all reward distributions are Dirac measures), where the policy is supposed to perform well. We show that its expected regret is of order $\frac{\rho\log n}{\Delta}$ (Proposition \ref{UCBsimplecase} for the upper bound and Proposition \ref{lbUCBsimplecase} for the lower bound).\\
These good performances are compensated by poor results in complexer environments, as we then prove that the overall expected regret lower bound is roughly of order $n^{1-2\rho}$ (Theorem \ref{th:ubr}).

\begin{prop}\label{UCBsimplecase}
Let $0\le b<a\le1$ and $n\ge 1$. For $\theta=(\delta_a,\delta_b)$,
the random variable $T_2(n)$ is uniformly upper bounded by $\frac{\rho}{\Delta^2}\log(n)+1$.
Consequently, the expected regret of {\sc ucb}($\rho$) is upper bounded by $\frac{\rho}{\Delta}\log(n)+1$.
\end{prop}
\begin{proof}
Let us prove the upper bound on the sampling time of the suboptimal arm by contradiction.
The assertion is obviously true for $n=1$ and $n=2$.
If the assertion is false, then there exists $t\ge 3$ such that
$T_2(t)>\frac{\rho}{\Delta^2}\log(t)+1$ and 
$T_2(t-1)\le \frac{\rho}{\Delta^2}\log(t-1)+1$.
Since $\log(t)\ge \log(t-1)$, this leads to $T_2(t)>T_2(t-1)$, meaning that arm $2$ is drawn at time $t$.
Therefore, we have
	$a+\sqrt{\frac{\rho \log(t)}{t-1-T_2(t-1)}} \le b+\sqrt{\frac{\rho \log(t)}{T_2(t-1)}}$,
hence
	$\Delta \le \sqrt{\frac{\rho \log(t)}{T_2(t-1)}}$,
which implies
	$T_2(t-1) \le \frac{\rho \log(t)}{\Delta^2}$ and thus
	$T_2(t) \le \frac{\rho \log(t)}{\Delta^2}+1$.
This contradicts the definition of $t$, which ends the proof of the first statement.
The second statement is a direct consequence of Formula \eqref{er}.
\end{proof}

The following shows that Proposition \ref{UCBsimplecase} is tight and allows to conclude that
the expected regret of {\sc ucb}($\rho$) is equivalent to 
	$\frac{\rho}{\Delta}\log(n)$
when $n$ goes to infinity.

\begin{prop}\label{lbUCBsimplecase}
Let $0\le b<a\le1$, $n\ge2$ and $h:t\mapsto \frac{\rho}{\Delta^2}\log(t)\Big(1+\sqrt{\frac{2\rho\log(t)}{(t-1)\Delta^2}}\Big)^{-2}$. For $\theta=(\delta_a,\delta_b)$,
the random variable $T_2(n)$ is uniformly lower bounded by 
	$$f(n)=\int_2^n \min\big(h'(s),1\big)ds - h(2).$$
As a consequence, the expected regret of {\sc ucb}($\rho$) is lower bounded by $\Delta f(n)$.
\end{prop}
Straightforward calculations shows that $h'(s)\leq 1$ for $s$ large enough, and this explains why our lower bound $\Delta f(n)$ is equivalent to $\Delta h(n)\sim\frac{\rho}{\Delta}\log(n)$ as $n$ goes to infinity.\\

\begin{proof}
First, one can easily prove (for instance, by induction) that $T_2(t) \le t/2$ for any $t\ge2$.
Let us prove the lower bound on $T_2(n)$ by contradiction. The assertion is obviously true for $n=2$.
If the assertion is false for $n\geq 3$, then there exists $t\ge 3$ such that
$T_2(t)<f(t)$ and 
$T_2(t-1)\ge f(t-1)$.
Since $f'(s) \in[0,1]$ for any $s\ge 2$, we have
$f(t)\le f(t-1)+1$. These last three inequalities imply $T_2(t)<T_2(t-1)+1$, which gives $T_2(t)=T_2(t-1)$. This means that arm $1$ is drawn at time $t$.
We consequently have
	$$a+\sqrt{\frac{\rho \log(t)}{t-1-T_2(t-1)}} \ge b+\sqrt{\frac{\rho \log(t)}{T_2(t-1)}},$$
hence
	$$\frac{\Delta}{\sqrt{\rho \log(t)}} \ge \frac1{\sqrt{T_2(t-1)}}-\frac1{\sqrt{t-1-T_2(t-1)}} \ge \frac1{\sqrt{T_2(t-1)}}-\frac{\sqrt2}{\sqrt{t-1}}.$$
We then deduce that $T_2(t) = T_2(t-1) \ge h(t) \ge f(t)$. This contradicts the definition of $t$, which ends the proof of the first statement.
Again, the second statement results from Formula \eqref{er}.
\end{proof}

Now we show that the order of the lower bound of the expected regret is $n^{1-2\rho}$. Thus, for $\rho\in(0,\frac{1}{2})$, {\sc ucb}($\rho$) does not perform enough exploration to achieve the logarithmic bound, as opposed to {\sc ucb}($\rho$) with $\rho\in(\frac{1}{2},+\infty)$.

\begin{theo}\label{th:ubr}
For any $\rho\in(0,\frac{1}{2})$, any $\theta\in\Theta$ and any $\beta\in(0,1)$, one has
$$
\mathbb{E}_\theta[R_n] \leq \sum_{k:\Delta_k>0}\frac{4\log n}{\Delta_k}+2\Delta_k\left(\frac{\log n}{\log(1/\beta)}+1\right)\frac{n^{1-2\rho\beta}}{1-2\rho\beta}.
$$
Moreover, for any $\varepsilon>0$, there exists $\theta\in\Theta$ such that 
$$\lim_{n\to+\infty}\frac{\E_{\theta}[R_n]}{n^{1-2\rho-\varepsilon}}=+\infty.$$
\end{theo}

\begin{proof}
Let us first show the upper bound. The core of the proof is a peeling argument and makes use of Hoeffding's maximal inequality. The idea is originally taken from \citet{audibert2009exploration}, and the following is an adaptation of the proof of an upper bound in the case $\rho>\frac{1}{2}$ which can be found in S. Bubeck's PhD dissertation.\\
First, let us notice that the policy selects arm $k$ such that $\Delta_k>0$ at step $t$ only if at least one of the three following equations holds:
\begin{eqnarray}
B_{k^*,T_{k^*}(t-1),t}\leq \mu^*,  \label{eq1}\\  
\hat X_{k,t}\geq \mu_k+\sqrt{\frac{\rho\log t}{T_k(t-1)}}, \label{eq2}\\
T_k(t-1)<\frac{4\rho\log n}{\Delta_k^2}. \label{eq3}
\end{eqnarray}
Indeed, if none of the equations holds, then:
$$
B_{k^*,T_{k^*}(t-1),t}>\mu^*=\mu_k+\Delta_k\geq \mu_k+2\sqrt{ \frac{\rho \log n}{T_k(t-1)}}> \hat X_{k,t}+\sqrt{ \frac{\rho \log t}{T_k(t-1)}}=B_{k,T_{k}(t-1),t}.
$$
We denote respectively by $\xi_{1,t}, \xi_{2,t}$ and $\xi_{3,t}$ the events corresponding to Equations  \eqref{eq1}, \eqref{eq2} and \eqref{eq3}.\\
We have:
\begin{eqnarray*}
\E_{\theta}[T_k(n)]&=&\E\left[\sum_{t=1}^n \IND_{I_t=k} \right]\leq\frac{4\log n}{\Delta_k^2}+\E\left[\sum_{t=\lceil 4\log n/\Delta_k^2 \rceil}^n \IND_{\{{I_t=k}\}\setminus\xi_{3,t}} \right]\\
&\leq&\frac{4\log n}{\Delta_k^2}+\E\left[\sum_{t=\lceil 4\log n/\Delta_k^2 \rceil}^n \IND_{\xi_{1,t}\cup \xi_{2,t}} \right]\leq \frac{4\log n}{\Delta_k^2}+\sum_{t=\lceil 4\log n/\Delta_k^2 \rceil}^n \Pro(\xi_{1,t})+\Pro(\xi_{2,t}). \\
\end{eqnarray*}

We now have to find a proper upper bound for $\Pro(\xi_{1,t})$ and $\Pro(\xi_{2,t})$. To this aim, we apply the peeling argument with a geometric grid over the time interval [1, t]:
\begin{eqnarray*}
\Pro(\xi_{1,t})&\leq& \Pro\left(\exists s\in \{1,\cdots,t\}, \ \hX_{k^*,s}+\sqrt{\frac{\rho\log t}{s}}\leq \mu^*\right)\\
&\leq& \sum_{j=0}^{\frac{\log t}{\log(1/\beta)}} \Pro\left(\exists s: \{\beta^{j+1}t<s\leq\beta^jt\}, \ \hX_{k^*,s}+\sqrt{\frac{\rho\log t}{s}}\leq \mu^*\right)\\
&\leq& \sum_{j=0}^{\frac{\log t}{\log(1/\beta)}} \Pro\left(\exists s: \{\beta^{j+1}t<s\leq\beta^jt\}, \ \sum_{l=1}^s X_{k^*,l}-\mu^*\leq-\sqrt{\rho s\log t}\right)\\
&\leq& \sum_{j=0}^{\frac{\log t}{\log(1/\beta)}} \Pro\left(\exists s: \{\beta^{j+1}t<s\leq\beta^jt\}, \ \sum_{l=1}^s X_{k^*,l}-\mu^*\leq-\sqrt{\rho \beta^{j+1}t\log t}\right).
\end{eqnarray*}
By means of Hoeffding-Azuma’s inequality for martingales, we then have:
\begin{eqnarray*}
\Pro(\xi_{1,t})&\leq& \sum_{j=0}^{\frac{\log t}{\log(1/\beta)}} \exp\left(-\frac{2\left(\sqrt{\beta^{j+1}t\rho\log t}\right)^2}{\beta^jt}\right)=\left(\frac{\log t}{\log(1/\beta)}+1\right)\frac{1}{t^{2\rho\beta}},
\end{eqnarray*}
and, for the same reasons, this bound also holds for $\Pro(\xi_{2,t})$.\\
Combining the former inequalities, we get:
\begin{eqnarray}
\E_{\theta}[T_k(n)]&\leq& \frac{4\log n}{\Delta_k^2}+2\sum_{t=\lceil 4\log n/\Delta_k^2 \rceil}^n \left(\frac{\log t}{\log(1/\beta)}+1\right)\frac{1}{t^{2\rho\beta}} \label{rb}\\
&\leq& \frac{4\log n}{\Delta_k^2}+2\left(\frac{\log n}{\log(1/\beta)}+1\right)\sum_{t=\lceil 4\log n/\Delta_k^2 \rceil}^n \frac{1}{t^{2\rho\beta}} \nonumber\\
&\leq&\frac{4\log n}{\Delta_k^2}+2\left(\frac{\log n}{\log(1/\beta)}+1\right)\int_1^n \frac{1}{t^{2\rho\beta}}dt\nonumber\\
&\leq&\frac{4\log n}{\Delta_k^2}+2\left(\frac{\log n}{\log(1/\beta)}+1\right)\frac{n^{1-2\rho\beta}}{1-2\rho\beta}.\nonumber
\end{eqnarray}
As usual, the bound on the expected regret then comes formula \eqref{er}.\\

Now let us show the lower bound. The result is obtained by considering an environment $\theta$ of the form $\left(Ber(\frac{1}{2}),\delta_{\frac{1}{2}-\Delta}\right)$, where $\Delta>0$ is such that $2\rho(1+\sqrt{\Delta})^2<2\rho+\varepsilon$. We set $T_n:=\lceil\frac{\rho\log n}{\Delta}\rceil$, and define the event $\xi_n$ by:
$$
\xi_n=\left\{\hX_{1,T_n}< \frac{1}{2}-(1+\frac{1}{\sqrt{\Delta}})\Delta\right\}.
$$
When event $\xi_n$ occurs, for any $t\in\set{T_n}{n}$ one has
\begin{eqnarray*}
\hat{X}_{1,T_n}+\sqrt{\frac{\rho\log t}{T_n}} & \leq & \hat{X}_{1,T_n}+\sqrt{\frac{\rho\log n}{T_n}}< \frac{1}{2}-(1+\frac{1}{\sqrt{\Delta}})\Delta+\sqrt{\Delta}\\
 & \leq & \frac{1}{2}-\Delta,
\end{eqnarray*}
so that arm 1 is chosen no more than $T_n$ times by {\sc ucb}($\rho$) policy. Thus:
$$
\E_\theta\left[T_2(n)\right]\geq \Pro_\theta(\xi_n)(n-T_n).
$$
We shall now find a lower bound of the probability of $\xi_n$ thanks to Berry-Esseen inequality. We denote by $C$ the corresponding constant, and by $\Phi$ the c.d.f. of the standard normal distribution. For convenience, we also define the following quantities:
$$
\sigma:=\sqrt{\E\left[\left(X_{1,1}-\frac{1}{2}\right)^2\right]}= \frac{1}{2}, \ M_3:=\E\left[\left|X_{1,1}-\frac{1}{2}\right|^3\right]=\frac{1}{8}.
$$
Using the fact that $\Phi(-x)=\frac{e^{-\frac{x^2}{2}}}{\sqrt{2\pi}x}\beta(x)$ with $\beta(x)\xrightarrow[x\to+\infty]{}1$, we are then able to write: 
\begin{eqnarray*}
\Pro_\theta(\xi_n) & = &\Pro_\theta\left(\frac{\hat{X}_{1,T_n}-\frac{1}{2}}{\sigma} \sqrt{T_n} \leq -2\left(1+\frac{1}{\sqrt{\Delta}}\right)\Delta\sqrt{T_n}\right) \\
& \geq & \Phi\left(-2(\Delta+\sqrt{\Delta})\sqrt{T_n}\right) - \frac{CM_3}{\sigma^3\sqrt{T_n}}\\
& \geq & \frac{\exp\left(-2(\frac{\rho\log n}{\Delta}+1)(\Delta+\sqrt{\Delta})^2\right)}{2\sqrt{2\pi}(\Delta+\sqrt{\Delta})\sqrt{T_n}}\beta\left(2(\Delta+\sqrt{\Delta})\sqrt{T_n}\right)- \frac{CM_3}{\sigma^3\sqrt{T_n}}\\
& \geq & n^{-2\rho(1+\sqrt{\Delta})^2}\frac{\exp\left(-2(\Delta+\sqrt{\Delta})^2\right)}{2\sqrt{2\pi}(\Delta+\sqrt{\Delta})\sqrt{T_n}}\beta\left(2(\Delta+\sqrt{\Delta})\sqrt{T_n}\right)- \frac{CM_3}{\sigma^3\sqrt{T_n}}.
\end{eqnarray*}
Previous calculations and Formula \eqref{er} gives $$\E_\theta[R_n]=\Delta\E_\theta[T_2(n)]\geq\Delta\Pro_\theta(\xi_n)(n-T_n)$$ and the former inequality easily leads to the conclusion of the theorem.
\end{proof}

\section{Selectivity}
In this section, we address the problem of selectivity in multi-armed stochastic bandit models. By selectivity, we mean the ability to adapt to the environment as and when rewards are observed. More precisely, it refers to the existence of a procedure that would perform at least as good as the policy that is best suited to the current environment $\theta$ among a given set of two (or more) policies. Two mains reasons motivates this study.\\
On the one hand this question was answered by \citeauthor{BurKat96} within the class of consistent policies. Let us recall the definition of consistent policies.\\
\begin{defi}
A policy is consistent if
$$
\forall a>0, \ \forall \theta\in\Theta, \ \E_\theta[R_n]=o(n^a).
$$
\end{defi}
Indeed they show the existence of lower bounds on the expected regret (see Section 3, Theorem 1 of \citet{BurKat96}), which we also recall for the sake of completeness.
\begin{theo}\label{th:BK}
Fix a consistent policy and $\theta\in\Theta$. If $k$ is a suboptimal arm (i.e. $\mu_k<\mu^*$) and if $0<D_k(\theta)<+\infty$, then
$$
\forall \varepsilon>0, \ \lim_{n\to+\infty}\Pro_\theta\left[T_k(n)\geq \frac{(1-\varepsilon)\log n}{D_k(\theta)}\right]=1.
$$
Consequently
$$
\liminf_{n\to+\infty} \frac{\E_{\theta}[T_k(n)]}{\log n}\geq\frac{1}{D_k(\theta)}.
$$
\end{theo}
Remind that the lower bound on the expected regret is then deduced from formula \eqref{er}.\\
\citeauthor{BurKat96} then exhibits an asymptotically optimal policy, i.e. which achieves the former lower bounds. The fact that a policy does as best as any other one obviously solves the problem of selectivity.\\
Nevertheless one can wonder what happens if we do not restrict our attention to consistent policies any more. Thus, one natural way to relax the notion of consistency is the following.
\begin{defi}
A policy is $\ua$-consistent if
$$
\forall a>\ua, \ \forall \theta\in\Theta, \ \E_\theta[R_n]=o(n^a).
$$
\end{defi}
For example we showed in the former section that {\sc ucb}($\rho$) is ($1-2\rho$)-consistent for any $\rho\in(0,\frac{1}{2})$. The class of $\alpha$-consistent policies will be studied in Section \ref{alpha}.\\
Moreover, as the expected regret of any policy is at most of order $n$, it seems simpler and relevant to only require it to be $o(n)$:
$$\forall \theta\in\Theta, \ \E_\theta[R_n]=o(n),$$
which corresponds to the definition of Hannan consistency. The class of Hannan consistent policies includes consistent policies and $\alpha$-consistent policies for any $\alpha\in(0,1)$. Some results on Hannan consistency will be provided in Section \ref{hannan}.\\

On the other hand, this problem has already been studied in the context of adversarial bandit by \citet{auer2003nonstochastic}. Their setting differs from our not only because their bandits are nonstochastic, but also because their adaptive procedure takes only into account a given number of recommendations, whereas in our setting the adaptation is supposed to come from observing rewards of the chosen arms (only one per time step). Nevertheless, there are no restrictions about consistency in the adversarial context and one can wonder if an "exponentially weighted forecasters" procedure like {\sc Exp4} could be transposed to our context. The answer is negative, as stated in the following theorem.

\begin{theo}
Let $\tilde{\mathcal A}$ be a consistent policy and let $\rho$ be a real in $(0,0.4)$. There are no policy which can both beat $\tilde{\mathcal A}$ and {\sc ucb}($\rho$), i.e.:
$$
\forall A, \ \exists \theta\in\Theta, \ \limsup_{n\to+\infty} \frac{\E_\theta[R_n(A)]}{\min(\E_\theta[R_n(\tilde{\mathcal A})],\E_\theta[R_n(\textit{\sc ucb}(\rho))])}>1.
$$
\end{theo}

Thus there are no optimal policy if we extend the notion of consistency. Precisely, as {\sc ucb}($\rho$) is $(1-2\rho)$-consistent, we have shown that there are no optimal policy within the class of $\ua$-consistent policies (which is included in the class of Hannan consistent policies), where $\ua>0.2$.\\
Moreover, ideas from selectivity in adversarial bandits can not work in the present context. As we said, this impossibility may also come from the fact that we can not observe at each step the decisions and rewards of more than one algorithm. Nevertheless, if we were able to observe a given set policies from step to step, then it would be easy to beat them all: it is then sufficient to aggregate all the observations and simply pull the arm with the greater empiric mean. The case where we only observe decisions (and not rewards) of a set of policies may be interesting, but is left outside of the scope of this paper.\\

\begin{proof}
Assume by contradiction that
$$
\exists A, \ \forall \theta\in\Theta, \ \limsup_{n\to+\infty} u_{n,\theta}\leq 1,
$$
where $u_{n,\theta}=\frac{\E_\theta[R_n(A)]}{\min(\E_\theta[R_n(\tilde{\mathcal A})],\E_\theta[R_n(UCB(\rho))])}$.\\
One has 
$$\E_\theta[R_n(A)]\leq u_{n,\theta}\E_\theta[R_n(\tilde{\mathcal A})],$$
so that the fact that $\tilde{\mathcal A}$ is a consistent policy implies that $A$ is also consistent. Consequently the lower bound of \citeauthor{BurKat96} has to hold. In particular, in environment $\theta=(\delta_0,\delta_\Delta)$ one has for any $\varepsilon>0$ and with positive probability (provided that $n$ is large enough):
$$
T_1(n)\geq \frac{(1-\varepsilon)\log n}{D_k(\theta)}.
$$
Now, note that there is simple upper bound of $D_k(\theta)$:
\begin{eqnarray*}
D_k(\theta)&\leq&\inf_{p,a\in[0,1]:  (1-p)a>\Delta}KL(\delta_0,p \delta_0+(1-p)\delta_a)\\
&=& \inf_{p,a\in[0,1]:  (1-p)a>\Delta} \log\left(\frac{1}{p}\right)=\log\left(\frac{1}{1-\Delta}\right).
\end{eqnarray*}
And on the other hand, one has by means of  Proposition \ref{lbUCBsimplecase}:
$$
T_1(n)\leq 1+\frac{\rho\log n}{\Delta^2}.
$$
Thus we have that, for any $\varepsilon>0$ and if $n$ is large enough
$$
1+\frac{\rho\log n}{\Delta^2}\geq \frac{(1-\varepsilon)\log n}{\log\left(\frac{1}{1-\Delta}\right)}
$$
Letting $\varepsilon$ go to zero and $n$ to infinity, we get:
$$
\frac{\rho}{\Delta^2}\geq \frac{1}{\log\left(\frac{1}{1-\Delta}\right)}.
$$
This means that $\rho$ has to be lower bounded by $ \frac{\Delta^2}{\log\left(\frac{1}{1-\Delta}\right)}$, but this is greater than $0.4$ if $\Delta=0.75$, hence the contradiction.


\end{proof}

Note that the former proof give us a simple alternative to Theorem \ref{th:ubr} to show that {\sc ucb}($\rho$) is not consistent if $\rho\leq 0.4$. Indeed if it were consistent, then in environment $\theta=(\delta_0,\delta_\Delta)$, $T_1(n)$ would also have to be greater than $\frac{(1-\varepsilon)\log n}{D_k(\theta)}$ and lower than $1+\frac{\rho\log n}{\Delta^2}$ , and the same contradiction would hold.

\section{Bounds on $\ua$-consistent policies}\label{alpha}
We now study $\ua$-consistent policies. We first show that the main result of \citeauthor{BurKat96} (Theorem \ref{th:BK}) can be extended in the following way.
\begin{theo}
Fix an $\ua$-consistent policy and $\theta\in\Theta$. If $k$ is a suboptimal arm and if $0<D_k(\theta)<+\infty$, then
$$
\forall \varepsilon>0, \ \lim_{n\to+\infty}\Pro_\theta\left[T_k(n)\geq (1-\varepsilon)\frac{(1-\ua)\log n}{D_k(\theta)}\right]=1.
$$
Consequently
$$
\liminf_{n\to+\infty} \frac{\E_{\theta}[T_k(n)]}{\log n}\geq\frac{1-\ua}{D_k(\theta)}.
$$
\end{theo}
Recall that, as opposed to \citet{BurKat96}, there are no optimal policy (i.e. a policy that would achieve the lower bounds in all environment $\theta$), as proven in the former section.\\

\begin{proof}
We adapt Proposition 1 in \citet{BurKat96} and its proof, which one may have a look at for further details. We fix $\varepsilon>0$, and we want to show that:
$$
\lim_{n\to+\infty} \Pro_\theta\left(\frac{T_k(n)}{\log n}\geq \frac{(1-\varepsilon)(1-\ua)}{D_k(\theta)} \right)=0.
$$
Set $\delta>0$ and $\delta'>\alpha$ such that $\frac{1-\delta'}{1+\delta}>(1-\varepsilon)(1-\alpha)$. By definition of $D_k(\theta)$, there exists $\tth$ such that $\E_{\tth}[X_{k,1}]>\mu^*$ and
$$
D_k(\theta)<KL(\nu_k,\tnu_k)<(1+\delta)D_k(\theta),\footnote{In \cite{BurKat96},  $D_k(\theta)$ is denoted ${\bf K}_a(\uuline{\theta})$ and $KL(\nu_k,\tnu_k)$ is denoted $\bf{I}(\underline{\theta}_a,\underline{\theta}_a')$.  The equivalence between other notations is straightforward. }
$$
where we denote $\theta=(\nu_1,\ldots,\nu_K)$ and $\tth=(\tnu_1,\ldots,\tnu_K)$.\\
Let us define $I^\delta=KL(\nu_k,\tnu_k)$ and the sets $$A_n^{\delta'}:=\left\{\frac{T_k(n)}{\log n}<\frac{1-\delta'}{I^\delta}\right\}, \ C_n^{\delta''}:=\left\{\log L_{T_k(n)}\leq \left(1-\delta''\right)\log n\right\},$$
where $\delta''$ is such that $\alpha<\delta''<\delta'$ and $L_j$ is defined by $\log L_j=\sum_{i=1}^{j} \log\left(\frac{d\nu_k}{d\tnu_k}(X_{k,i})\right)$.\\
We show that $\Pro_\theta(A_n^{\delta'})=\Pro_\theta(A_n^{\delta'}\cap C_n^{\delta''})+\Pro_\theta(A_n^{\delta'}\setminus C_n^{\delta''}) \xrightarrow[n\to+\infty]{}0$.\\
On the one hand, one has:
\begin{eqnarray}
\Pro_\theta(A_n^{\delta'}\cap C_n^{\delta''})&\leq& e^{(1-\delta'')\log n}\Pro_{\tth}(A_n^{\delta'}\cap C_n^{\delta''})\label{defc}\\
&\leq& n^{1-\delta''}\Pro_{\tth}(A_n^{\delta'})=  n^{1-\delta''}\Pro_{\tth}\left(n-T_k(n)>n-\frac{1-\delta'}{I^\delta}\log n\right)\nonumber\\
&\leq&\frac{n^{1-\delta''}\E_{\tth}[n-T_k(n)]}{n-\frac{1-\delta'}{I^\delta}\log n} \label{markov} \\
&\leq&\frac{\sum_{l\neq k}n^{-\delta''}\E_{\tth}[T_l(n)]}{1-\frac{1-\delta'}{I^\delta}\frac{\log n}{n}} \xrightarrow[n\to+\infty]{}0, \nonumber
\end{eqnarray}
where \eqref{defc} is consequence of the definition of $C_n^{\delta''}$, \eqref{markov} comes from Markov's inequality, and where the final limit is a consequence of the $\alpha$-consistence.\\
On the other hand we set $b_n:=\frac{1-\delta'}{I^\delta}\log n$, so that we have:
\begin{eqnarray*}
\Pro_\theta(A_n^{\delta'}\setminus C_n^{\delta''})&\leq& \Pro\left(\max_{j\leq\lfloor b_n\rfloor}\log L_j>(1-\delta'')\log n\right)\\
&\leq& \Pro\left(\frac{1}{b_n}\max_{j\leq\lfloor b_n\rfloor}\log L_j>I^\delta\frac{1-\delta''}{1-\delta'}\right).
\end{eqnarray*}
This term then tends to zero, as a consequence of the law of large numbers.\\
Now that $\Pro_\theta(A_n^{\delta'})$ tends to zero,  the conclusion comes from the following inequality:
$$\frac{1-\delta'}{I^\delta}>\frac{1-\delta'}{(1+\delta)D_k(\theta)}\geq \frac{(1-\varepsilon)(1-\ua)}{D_k(\theta)} .$$

\end{proof}

The former lower bound is asymptotically optimal, as claimed in the following proposition.

\begin{prop}
There exists $\theta\in \Theta$ and a constant $c>0$ such that, for any $\alpha\in[0,1)$, there exists an $\alpha$-consistent policy and $k\neq k^*$ such that:
$$
\liminf_{n\to+\infty} \frac{\E_\theta[T_k(n)]}{(1-\alpha)\log n}\leq c.
$$ 
\end{prop}

\begin{proof}
By means of Proposition \ref{UCBsimplecase}, the following holds for {\sc ucb}($\rho$) in any environment of the form $\theta=(\delta_a,\delta_b)$ with $a\neq b$:
$$
\liminf_{n\to +\infty} \frac{\E_\theta T_k(n)}{\log n} \leq \frac{\rho}{\Delta^2},
$$
where $k\neq k^*$.\\ As {\sc ucb}($\rho$) is $(1-2\rho)$-consistent (Theorem \ref{th:ubr}), we can conclude by setting $c=\frac{1}{2\Delta^2}$ and by choosing the policy {\sc ucb}($\frac{1-\ua}{2}$).


\end{proof}

\section{Hannan consistency and other exploration functions}\label{hannan}
We now study the class of Hannan consistent policies. We first show the necessity to have a logarithmic lower bound in some environments $\theta$, and then a study of extended {\sc ucb} policies will prove that there does not exist a logarithmic bound on the whole set $\Theta$.

\subsection{The necessity of a logarithmic regret in some environments}

A simple idea enables to understand the necessity of a logarithmic regret in some environments. Assume that the agent knows the number of rounds $n$, and that he balances exploration and exploitation in the following way: he first pulls each arm $s(n)$ times, and then selects the arm that has obtained the best empiric mean for the rest of the game. 
If we denote by $p_{s(n)}$ the probability that the best arm does not have the best empiric mean after the exploration phase (i.e. after the first $Ks(n)$ rounds), then the expected regret is of the form
\begin{equation}\label{eq:idea}
c_1(1-p_{s(n)})s(n) + c_2p_{s(n)}n.
\end{equation}
Indeed if the agent manages to match the best arm then he only suffers the pulls of suboptimal arms during the exploration phase, and that represents an expected regret of order $s(n)$. If not, the number of pulls of suboptimal arms is of order $n$, and so is the expected regret.\\
Now we can approximate $p_{s(n)}$, because it has the same order as the probability that the best arm gets an empiric mean lower than the second best mean reward, and because $\frac{X_{k^*,s(n)}-\mu^*}{\sigma}\sqrt{s(n)}$ (where $\sigma$ is the variance of $X_{k^*,1}$) approximately has a standard normal distribution by the central limit theorem:
\begin{eqnarray*}
p_{s(n)}&\approx&\P_\theta(X_{k^*,s(n)}\leq \mu^*-\Delta)=\P_\theta\left(\frac{X_{k^*,s(n)}-\mu^*}{\sigma}\sqrt{s(n)}\leq-\frac{\Delta\sqrt{s(n)}}{\sigma}\right)\\
&\approx&\frac{1}{\sqrt{2\pi}}\frac{\sigma}{\Delta\sqrt{s(n)}}\exp\left(-\frac{1}{2}\left(\frac{\Delta\sqrt{s(n)}}{\sigma}\right)^2\right)\\
&\approx&\frac{1}{\sqrt{2\pi}}\frac{\sigma}{\Delta\sqrt{s(n)}}\exp\left(-\frac{\Delta^2s(n)}{2\sigma^2}\right).
\end{eqnarray*}
Then it is clear why the expected regret has to be logarithmic: $s(n)$ has to be greater than $\log n$ if we want the second term $p_{s(n)}n$ of Equation \eqref{eq:idea} to be sub-logarithmic, but then first term $(1-p_{s(n)})s(n)$ is greater than $\log n$.\\

This idea can be generalized, and this gives the following proposition. 

\begin{prop}
For any policy, there exists $\theta\in \Theta$ and such that 
$$ 
\limsup_{n\to+\infty} \frac{\E_\theta R_n}{\log n}>0.
$$
\end{prop}

This result can be seen as a consequence of the main result of \citeauthor{BurKat96} (Theorem \ref{th:BK}): if we assume by contradiction that $\limsup_{n\to+\infty} \frac{\E_\theta R_n}{\log n}=0$ for all $\theta$, the considered policy is therefore consistent, but then the logarithmic lower bounds have to hold. The reason why we wrote the proposition anyway is that our proof is based on the simple reasoning stated above and that it consequently holds beyond our model (see the following for details).\\

\begin{proof}
The proposition results from the following property on $\Theta$: there exists two environments $\theta=(\nu_1,\ldots,\nu_K)$ and $\tilde{\theta}=(\tnu_1,\ldots,\tnu_K)$ and $k\in\set{1}{K}$ such that
\begin{itemize}
\item $k$ has the best mean reward in environment $\theta$,
\item $k$ is not the winning arm in environment $\tth$,
\item $\nu_k=\tnu_k$ and there exists $\eta\in(0,1)$ such that  
\begin{equation}\label{aeta}
\prod_{\ell\neq k}\frac{d\nu_\ell}{d\tnu_\ell}(X_{\ell,1})\geq \eta \ \ \P_{\tth}- a.s.
\end{equation}
\end{itemize}

The idea is the following: in case $\nu_k=\tnu_k$ is likely to be the reward distribution of arm $k$, then arm $k$ has to be pulled often for the regret to be small if the environment is $\theta$, but not so much, as one has to explore to know if the environment is actually $\tth$ (and the third condition ensures that the distinction can be tough to make). The lower bound on exploration is of order $\log n$, as in the sketch in the beginning of the section.

The proof actually holds for any $\Theta$ that has the above-mentioned property (i.e. without the assumptions we made on $\Theta$, i.e. being of the form $\Theta_1\times\ldots\times\Theta_K$ and/or containing distributions of the form $p\delta_a+(1-p)\delta_b$). In our setting, the property is easy to check. Indeed the three conditions hold for any $k$ and any pair of environments $\theta=(\nu_1,\ldots,\nu_K)$, $\tilde{\theta}=(\tnu_1,\ldots,\tnu_K)$ such that each $\nu_\ell$ (resp. $\tnu_\ell$) is a Bernoulli law with parameter $p_\ell$ (resp. $\tilde{p}_\ell$) and such that
\begin{itemize}
\item $\forall \ell\neq k, \ \tilde{p}_k>\tilde{p}_\ell$,
\item $\exists \ell\neq k, \  p_k<p_\ell$,
\item $\tilde{p}_k=p_k$ and $p_\ell, \ \tilde{p}_\ell\in(0,1)$ for any $\ell\neq k$.
\end{itemize}
It is then sufficient to set $$\eta=\left(\min\left\{\frac{p_1}{\tilde{p}_1},\ldots,\frac{p_{k-1}}{\tilde{p}_{k-1}},\frac{p_{k+1}}{\tilde{p}_{k+1}},\ldots,\frac{p_K}{\tilde{p}_K},\frac{1-p_1}{1-\tilde{p}_1},\ldots,\frac{1-p_{k-1}}{1-\tilde{p}_{k-1}},\frac{1-p_{k+1}}{1-\tilde{p}_{k+1}},\ldots,\frac{1-p_K}{1-\tilde{p}_K}\right\}\right)^{K-1},$$
as
$\frac{d\nu_\ell}{d\tnu_\ell}(X_{\ell,1})$ equals $\frac{p_\ell}{\tilde{p}_\ell}$ when $X_{\ell,1}=1$ and $\frac{1-p_\ell}{1-\tilde{p}_\ell}$ when $X_{\ell,1}=0$.\\

We will now compute a lower bound of the expected regret in environment $\tth$. To this aim, we set\\
$$
g(n):=\frac{2\E_\theta R_n}{\Delta}.
$$\\
 In the following, $\tilde{\Delta}_k$ denotes the optimality gap of arm $k$ in environment $\tth$. Moreover the switch from $\tth$ to $\theta$ will result from Equality \eqref{aeta} and from the fact that event $\left\{\sum_{\ell\neq k}T_\ell(n)\leq g(n) \right\}$ is measurable with respect to $X_{\ell,1},\ldots,X_{\ell,\lfloor g(n) \rfloor}$ ($\ell\neq k$) and to $X_{k,1},\ldots,X_{k,n}$. That enables us to introduce the function $q$ such that 
$$
\IND_{\left\{\sum_{\ell\neq k}T_\ell(n)\leq g(n) \right\}}=q\big((X_{k,s})_{s=1..n}, (X_{\ell,s})_{\ell\neq k, \ s=1..\lfloor g(n) \rfloor}\big)
$$
and to write:
\begin{eqnarray*}
\E_{\tth}R_n&\geq&\tilde{\Delta}_k\E_{\tth}[T_k(n)]\geq \tilde{\Delta}_k(n-g(n))\P_{\tth}\left(T_k(n)\geq n-g(n)\right)\\
&=& \tilde{\Delta}_k(n-g(n))\P_{\tth}\left(\sum_{\ell\neq k}T_\ell(n)\leq g(n)\right)\\
&=& \tilde{\Delta}_k(n-g(n))\int q\big((x_{\ell,s})_{\ell\neq k, \ s=1..t},(x_{k,s})_{s=1..n}\big)\hspace{-.5cm}\prod_{\tiny \begin{array}c \ell\neq k \nonumber\\
s=1..\lfloor g(n) \rfloor\end{array}}\hspace{-.5cm}d\tnu_\ell(x_{\ell,s})\prod_{s=1..n}d\tnu_k(x_{k,s})\\
&\geq& \tilde{\Delta}_k(n-g(n))\eta^{\lfloor g(n) \rfloor} \int q\big((x_{\ell,s})_{\ell\neq k, \ s=1..t},(x_{k,s})_{s=1..n}\big)\hspace{-.6cm}\prod_{\tiny \begin{array}c \ell\neq k \nonumber\\
s=1..\lfloor g(n) \rfloor\end{array}}\hspace{-.6cm}d\nu_\ell(x_{\ell,s})\prod_{s=1..n}d\nu_k(x_{k,s})\\
&\geq& \tilde{\Delta}_k(n-g(n))\eta^{g(n)}\P_{\theta}\left(\sum_{\ell\neq k}T_\ell(n)\leq g(n)\right)\\
&=& \tilde{\Delta}_k(n-g(n))\eta^{g(n)}
\left(1-\P_{\theta}\left(\sum_{\ell\neq k}T_\ell(n)> g(n)\right)
\right)\\
&\geq& \tilde{\Delta}_k(n-g(n))\eta^{g(n)}
\left(1-\frac{
\E_{\theta}\left(\sum_{\ell\neq k}T_\ell(n)\right)}{g(n)}
\right)
\\
&\geq& \tilde{\Delta}_k(n-g(n))\eta^{g(n)}
\left(1-\frac{
\E_{\theta}\left(\sum_{\ell\neq k}\Delta_\ell T_\ell(n)\right)}{\Delta g(n)}
\right)
\\
&\geq& \tilde{\Delta}_k(n-g(n))\eta^{g(n)}
\left(1-\frac{
\E_{\theta}R_n}{\Delta g(n)}
\right)= \tilde{\Delta}_k\frac{n-g(n)}{2}\eta^{g(n)},\\
\end{eqnarray*}
where the very first inequality is a consequence of Formula \eqref{er}.\\
We are now able to conclude. Indeed, if we assume that $\frac{\E_\theta R_n}{\log n}\xrightarrow[n\to+\infty]{}0$, then one has $g(n)\leq\min\left(\frac{n}{2},\frac{-\log n}{2\log \eta}\right)$ for $n$ large enough and:
$$
\E_{\tth}R_n\geq \tilde{\Delta}_k\frac{n-g(n)}{2}\eta^{g(n)}\geq \tilde{\Delta}_k\frac{n}{4}\eta^{\frac{-\log n}{2\log \eta}}=\tilde{\Delta}_k\frac{\sqrt{n}}{4}.
$$
In particular, we have $\frac{\E_{\tth} R_n}{\log n}\xrightarrow[n\to+\infty]{}+\infty$, hence the conclusion.

\end{proof}
To finish this section, note that a proof could have been written in the same way with a slightly different property on $\Theta$: there exists two environments $\theta=(\nu_1,\ldots,\nu_K)$ and $\tilde{\theta}=(\tnu_1,\ldots,\tnu_K)$ and $k\in\set{1}{K}$ such that
\begin{itemize}
\item $k$ has the best mean reward in environment $\theta$,
\item $k$ is not the winning arm in environment $\tth$,
\item $\nu_\ell=\tnu_\ell$ for all $\ell\neq k$ and there exists $\eta\in(0,1)$ such that  
\begin{equation*}
\frac{d\nu_k}{d\tnu_k}(X_{k,1})\geq \eta \ \ \P_{\tth}- a.s.
\end{equation*}
\end{itemize}
The dilemma is then between exploring arm $k$ or pulling the best arm of environment $\tth$.

\subsection{There are no logarithmic bound in general}
We extend our study to more general {\sc ucb} policies, and we will find that there does not exist logarithmic lower bounds of the expected regret in the case of Hannan consistency. With "{\sc ucb}", we now refer to an {\sc ucb} policy with indexes of the form:
$$
B_{k,s,t}= \hX_{k,s}+\sqrt{\frac{f_k(t)}{s}}
$$
where functions $f_k$ ($1\leq k\leq K$) are increasing.\\

To find conditions for Hannan consistency, let us first show the following upper bound.

\begin{lemm}\label{lemma}
If arm k does not have the best mean reward, then for any $\beta\in (0,1)$ the following upper bound holds:
$$
\E_\theta[T_k(n)]\leq u + \sum_{t=u+1}^n\left(1+\frac{\log t}{\log(\frac{1}{\beta})}\right)\left(e^{-2\beta f_k(t)}+e^{-2\beta f_{k^*}(t)}\right),
$$
where $u=\left\lceil\frac{4f_k(n)}{\Delta_k}\right\rceil$.
\end{lemm}

\begin{proof}
We adapt the arguments leading to Equation \eqref{rb} in the proof of Theorem \ref{th:ubr}. We begin by noticing that, if arm $k$ is selected, then at least one of the three following equations holds:
\begin{eqnarray*}
B_{k^*,T_{k^*}(t-1),t}\leq \mu^*,\\  
\hat X_{k,t}\geq \mu_k+\sqrt{\frac{f_k(t)}{T_k(t-1)}},\\
T_k(t-1)<\frac{4f_k(n)}{\Delta_k^2},
\end{eqnarray*}
and the rest follows straightforwardly.
\end{proof}

We are now able to give sufficient conditions on the $f_k$ for {\sc ucb} to be Hannan consistent.

\begin{prop}
If $f_k(n)=o(n)$ for all $k\in\set{1}{K}$, and if there exists $\gamma>\frac{1}{2}$ and $N\geq 1$ such that $f_k(n)\geq \gamma \log\log n$ for all $k\in\set{1}{K}$ and for any $n\geq N$, then {\sc ucb} is Hannan consistent.\\
\end{prop}

\begin{proof}
Fix an index $k$ of a suboptimal arm and choose $\beta\in(0,1)$ such that $2\beta\gamma>1$. By means of Lemma \ref{lemma}, one has for $n$ large enough:
$$
\E_\theta[T_k(n)]\leq u + 2\sum_{t=u+1}^n\left(1+\frac{\log t}{\log(\frac{1}{\beta})}\right)e^{-2\beta\gamma\log\log t},
$$
where $u=\left\lceil\frac{4f_k(n)}{\Delta_k}\right\rceil$.\\
Consequently, we have:
\begin{equation}\label{eq:...}
\E_\theta[T_k(n)]\leq u + 2\sum_{t=2}^n\left(\frac{1}{(\log t)^{2\beta\gamma}}+\frac{1}{\log(\frac{1}{\beta})}\frac{1}{(\log t)^{2\beta\gamma-1}}\right).
\end{equation}

Sums of the form $\sum_{t=2}^n\frac{1}{(\log t)^{c}}$ with $c>0$ are equivalent to $\frac{n}{(\log n)^{c}}$ as $n\to +\infty$. Indeed, on the one hand we have
$$
\sum_{t=3}^n\frac{1}{(\log t)^{c}} \leq \int_2^n \frac{dx}{(\log x)^{c}} \leq \sum_{t=2}^n\frac{1}{(\log t)^{c}},
$$
so that $\sum_{t=2}^n\frac{1}{(\log t)^{c}}\sim\int_2^n \frac{dx}{(\log x)^{c}}$. On the other hand, one can write
$$
\int_2^n \frac{dx}{(\log x)^{c}}=\left[\frac{x}{(\log x)^{c}}\right]_2^n+c\int_2^n \frac{dx}{(\log x)^{c+1}}.
$$
As both integrals are divergent we have $\int_2^n \frac{dx}{(\log x)^{c}}=o\left(\int_2^n \frac{dx}{(\log x)^{c+1}}\right)$, so that $\int_2^n \frac{dx}{(\log x)^{c}}\sim \frac{n}{(\log n)^{c}}$.\\

Now, by means of Equation \eqref{eq:...}, there exists $C>0$ such that
$$
\E_\theta[T_k(n)]\leq \left\lceil\frac{4f_k(n)}{\Delta}\right\rceil + \frac{Cn}{(\log n)^{2\beta\gamma-1}},
$$
and this proves Hannan consistency.
\end{proof}

The fact that there is no logarithmic lower bound then comes from the following proposition (which is a straightforward adaptation of Propostion \ref{UCBsimplecase}).

\begin{prop}
Let $0\le b<a\le1$ and $n\ge 1$. For $\theta=(\delta_a,\delta_b)$,
the random variable $T_2(n)$ is uniformly upper bounded by $\frac{f_2(n)}{\Delta^2}+1$.
Consequently, the expected regret of {\sc ucb} is upper bounded by $\frac{f_2(n)}{\Delta}+1$.
\end{prop}

Then, if $f_1(n)=f_2(n)=\log\log n$, {\sc ucb} is Hannan consistent and the expected regret is of order $\log\log n$ in all environments of the form $(\delta_a,\delta_b)$. Hence the conclusion on the non-existence of logarithmic lower bounds.

\bibliography{bandit_stat}

\end{document}